\documentclass[letterpaper, 10 pt, conference]{ieeeconf}  

\IEEEoverridecommandlockouts                              

\usepackage{amsmath} 
\usepackage{mathtools}
\usepackage{multicol}
\usepackage{subcaption}
\usepackage{multirow}
\usepackage{booktabs}
\usepackage{xurl}

\usepackage{pgfplots}
\usepackage{pgfplotstable}
\usepgfplotslibrary{groupplots}

\DeclareMathOperator{\slope}{slope}
\DeclareMathOperator{\blkdiag}{blkdiag}
\DeclareMathOperator{\chol}{chol}
\DeclareMathOperator{\diag}{diag}
\DeclareMathOperator{\circo}{circ}
\DeclareMathOperator{\veco}{vec}

\newtheorem{theorem}{\bf Theorem}
\newtheorem{definition}[theorem]{\bf Definition}
\newtheorem{lemma}[theorem]{\bf Lemma}
\newtheorem{remark}[theorem]{\bf Remark}

\newcommand{\bbR}{\mathbb{R}}
\newcommand{\bbS}{\mathbb{S}}

\usepackage{amsmath,amsfonts,bm}

\def\1{\bm{1}}

\def\mI{{\bm{I}}}

\def\mL{{\bm{L}}}
\def\mM{{\bm{M}}}

\def\mQ{{\bm{Q}}}
\def\mR{{\bm{R}}}
\def\mS{{\bm{S}}}

\def\mX{{\bm{X}}}

\DeclareMathAlphabet{\mathsfit}{\encodingdefault}{\sfdefault}{m}{sl}
\SetMathAlphabet{\mathsfit}{bold}{\encodingdefault}{\sfdefault}{bx}{n}

\title{\LARGE \bf
Neural network training under semidefinite constraints
}

\author{Patricia Pauli$^{1}$, Niklas Funcke$^{1}$, Dennis Gramlich$^{2}$, Mohamed Amine Msalmi$^{1}$ and Frank Allgöwer$^{1}$
\thanks{*This work was funded by Deutsche Forschungsgemeinschaft (DFG, German Research Foundation) under Germany's Excellence Strategy - EXC 2075 - 390740016 and under grant 468094890.}
\thanks{$^{1}$Patricia Pauli, Niklas Funcke, Mohamed Amine Msalmi and Frank Allgöwer are with the Institute for Systems Theory and Automatic Control, University of Stuttgart, 70569 Stuttgart, Germany
        {\tt\small patricia.pauli@ist.uni-stuttgart.de}}%
\thanks{$^{2}$Dennis Gramlich is with the Chair of Intelligent Control Systems, RWTH Aachen, 52074 Aachen, Germany}%
}

\begin{document}

\maketitle
\thispagestyle{empty}
\pagestyle{empty}

\begin{abstract}
This paper is concerned with the training of neural networks (NNs) under semidefinite constraints, which allows for NN training with robustness and stability guarantees. In particular, we focus on Lipschitz bounds for NNs. Exploiting the banded structure of the underlying matrix constraint, we set up an efficient and scalable training scheme for NN training problems of this kind based on interior point methods. Our implementation allows to enforce Lipschitz constraints in the training of large-scale deep NNs such as Wasserstein generative adversarial networks (WGANs) via semidefinite constraints. In numerical examples, we show the superiority of our method and its applicability to WGAN training.
\end{abstract}

\section{Introduction}
Neural networks (NNs) are successfully applied to a broad range of applications. Yet, due to the prevailing lack of guarantees and the black box nature of NNs, engineers still hesitate to use them in safety critical applications, making NN safety and stability an active field of research. Remedies in this respect include, e.g., constraint violation penalties, constraint enforcing architecture design, or data augmentation \cite{stewart2017label,marquez2017imposing}. Opposed to these approaches which lack rigorous mathematical guarantees, we propose an efficient and scalable training scheme for NNs with guarantees, i.e., the resulting NNs satisfy stability and robustness certificates that are captured by semidefinite constraints.

The use and analysis of NNs in control-specific problems, such as nonlinear system identification and controller design, has been promoted for a while \cite{suykens1995artificial,levin1993control}. Recently, \cite{fazlyab2020safety} suggested an analysis of NNs using semidefinite programming (SDP), based on the over-approximation of the underlying sector-bounded and slope-restricted activation functions using quadratic constraints (QCs). \cite{fazlyab2019efficient} used this approach to analyze the robustness of NNs by determining accurate upper bounds on their Lipschitz constant. Following the same ideas, \cite{yin2021stability,pauli2021offset,hashemi2021certifying,nikolakopoulou2020feedback}  addressed closed-loop stability of feedback systems that include an NN, e.g., as an NN controller. Besides in analysis, semidefinite constraints have been used for the training of Lipschitz-bounded NNs \cite{pauli2021training} and  recurrent neural networks \cite{revay2020convex,pauli2022robustness} and for imitation learning to approximate controllers \cite{yin2021imitation}. The present work aims to formalize NN training subject to semidefinite constraints and moreover, to provide an efficient and scalable training method for problems of that kind. Prior works like \cite{pauli2021training} and \cite{yin2021imitation} use the alternating directions method of multipliers (ADMM) for NN training under linear matrix inequality (LMI) constraints, leading to a significant increase in training time in comparison to the unconstrained problem. To improve on these existing training schemes, we make use of barrier functions to include the semidefinite constraints in the training loss, similar to \cite{revay2020convex}, and then train the resulting unconstrained problem via backpropagation. Recently, efforts were made to improve the scalability of SDP-based neural network verification methods exploiting the sparse structure of the certification matrix \cite{newton2021exploiting}. Similarly, in the specific case of enforcing Lipschitz bounds, we exploit the block-banded matrix structure to accelerate the training, yielding an efficient and robust training procedure with good scalability to deep NNs with convolutional layers, as e.g. used in Wasserstein generative adversarial networks (WGANs).

Generative adversarial networks (GANs) can generate fake data that are impressively similar to real-world data \cite{goodfellow2014generative}. However, original GAN training struggled with vanishing gradients and was hence developed further, yielding WGAN training \cite{arjovsky2017wasserstein}. Here, the Wasserstein distance is used instead of the Jenson-Shannon distance to measure the distance between the probabilities of the real and the fake data. The Kantorovich-Rubinstein duality renders the training objective computational, where an optimization problem over 1-Lipschitz continuous NNs is used to estimate the Wasserstein-1 distance. This Lipschitz constraint in WGAN training was first realized by weight clipping which is highly conservative, resulting in NNs with unnecessarily low Lipschitz constants. \cite{gulrajani2017improved} suggested a gradient penalty method instead that in turn provides no guarantees. Using a semidefinite constraint to enforce the Lipschitz constraint and applying our training scheme can guarantee the Lipschitz condition in a less conservative fashion, i.e., the underlying NNs have a Lipschitz constant just below $1$.

The contribution of this paper is twofold. On the one hand, we provide an efficient and scalable training scheme for NNs guaranteeing certificates described by semidefinite constraints, which we in particular show for the example of Lipschitz continuity. Herein, we exploit the banded structure of the matrices of these semidefinite constraints. On the other hand, we apply our training scheme to the popular example of WGAN training which extends the method to large-scale structures that include convolutional layers. The remainder of this paper is organized as follows. In Section~\ref{sec:problem} we state the problem setup. In Section~\ref{sec:training} we introduce the training scheme and finally, in Section~\ref{sec:experiments} we show the advantages of our method in comparison to prior work on Lipschitz-bounded NNs and we illustrate its applicability to WGAN training.

{\bf Notation:} $\mathbb{S}^{n}$ ($\mathbb{S}_{++}^{n}$) denotes the set of $n$-by-$n$ symmetric (positive definite) matrices and $\mathbb{D}_+^{n_i}$ denotes the set of diagonal matrices with nonnegative entries, i.e., $\mathbb{D}_+^{n_i}\coloneqq\{X\in\mathbb{R}^{n_i\times n_i}\mid X=\diag(\lambda),\lambda\in\mathbb{R}^{n_i},\lambda_i\geq 0\}$.

\section{Problem statement}\label{sec:problem}
In this paper, we address the NN training problem
\begin{equation}\label{eq:training_problem}
    \min_{\theta,\kappa} ~ \mathcal{L}(f_\theta) \quad \mathrm{s.\,t.} ~ \mM_j(\theta,\kappa)\succeq 0, \quad j=0,\dots,q.
\end{equation}
We assume that the parametric function $f_\theta:\bbR^{n_0}\to\bbR^{n_{l+1}}$ is a feedforward NN with $l$ hidden layers
\begin{equation}\label{eq:NN}
\begin{split}
w^0&=x,\\w^i& =\phi_{i}(W_{i-1}w^{i-1}+b_{i-1}),~i=1,\dots,l,\\
f_\theta(x)&=W_lw^l+b_l,
\end{split}
\end{equation}
where the parameter $\theta = (W_i,b_i)_{i = 0}^l$ collects all weight matrices $W_i \in \mathbb{R}^{n_{i+1}\times n_{i}}$ and all biases $b_i\in\mathbb{R}^{n_{i+1}}$. The function $\mathcal{L}(f_\theta)$ is the loss function of this training problem, e.g., the mean squared error or the cross-entropy loss with respect to some dataset, and the layerwise activation functions are denoted by $\phi_i :\mathbb{R}^{n_i}\to \mathbb{R}^{n_i}$. These notions are standard in deep learning except for the matrix inequality constraints $\mM_j(\theta,\kappa)\succeq 0$, $j=1,\dots,q$, where the parameter $\kappa$ holds any additional decision variables not already included in $\theta$.

\subsection{Semidefinite constraints in neural network training}\label{sec:examples}
In this section, we present LMI certificates for Lipschitz continuity for NNs and further introduce Wasserstein GANs. After introducing the specific examples, we establish certificates for general behavioral properties based on quadratic constraints and outline the underlying theory.

\subsubsection{Training of robust NNs}\label{sec:ex_robust}
The robustness of an NN is typically measured by its Lipschitz constant $L^*$ \cite{szegedy2013intriguing}, which is the smallest nonnegative number $L \in \mathbb{R}$ for which
\begin{equation*}\label{eq:Lip}
    \left\lVert f_\theta(x_1)-f_\theta(x_2)\right\rVert\leq L \left\rVert x_1-x_2\right\rVert\quad\forall x_1,x_2\in\mathbb{R}^{n_0}
\end{equation*}
or equivalently
\begin{equation*}\label{eq:qc_Lipschitz}
    \begin{bmatrix}
        x_1-x_2 \\
        f_\theta(x_1) - f_\theta(x_2) 
    \end{bmatrix}^\top
    \begin{bmatrix}
        L^2 \mI & 0 \\
        0 & -\mI
    \end{bmatrix}
    \begin{bmatrix}
        x_1-x_2 \\
        f_\theta(x_1) - f_\theta(x_2)
    \end{bmatrix}\geq 0
\end{equation*}
holds. A certificate for $L$-Lipschitz continuity of the NN $f_\theta$ is given by the semidefinite constraint $\mM(\theta,\kappa) =$
\begin{equation} \label{eq:lipschitz_matrix_inequality}
\begin{split}
    \left[\begin{array}{cccccc}
        \!L^2 \mI\!             &  \!\!-W_0^\top \Lambda_1\!\!  &   0          &   \cdots      &     0  \\
        \!\!-\Lambda_1 W_0\!\!    &  \!2\Lambda_1\!             &   \ddots         &  \ddots   &  \vdots           \\
        0            &  \ddots                      & \ddots            &  \!-W_{l-1}^\top \Lambda_l\!  & 0             \\
        \vdots        &   \ddots                 &  \!\!-\Lambda_l W_{l-1}\!\!     &  \!2\Lambda_l\phantom{\vdots}\!          & \!-W_l^\top\!     \\
        0                 &       \cdots   &  0                &  -W_l        & \phantom{\vdots}\mI\phantom{\vdots} 
    \end{array}\right]\!\succeq\!0,
\end{split}
\end{equation}
where $\kappa = (\Lambda_1,\ldots,\Lambda_l)$ with $\Lambda_i \in \mathbb{D}^{n_i}_+,~i=1,\dots,l$ \cite{pauli2021training}.

\subsubsection{Wasserstein GANs}\label{sec:ex_GANs}
Let $f_{\theta_1}$ and $g_{\theta_2}$ be two NNs, of which we call $f_{\theta_1}$ the discriminator and $g_{\theta_2}$ the generator. The WGAN training problem on a dataset $\mathcal{D} = \{x_1,\ldots,x_N\}$ is
\begin{subequations}\label{eq:WGAN_prob}
\begin{align}
    \min_{\theta_2} & \max_{\theta_1} ~\frac{1}{N} \sum_{x\in \mathcal{D}} f_{\theta_1}(x) - \mathbb{E}_{z\sim \mathcal{N}(0,\mI)} f_{\theta_1}(g_{\theta_2}(z))\\
    \mathrm{s.\,t.} & \left\lVert f_{\theta_1}(x_1)-f_{\theta_1}(x_2)\right\rVert\leq \left\rVert x_1-x_2\right\rVert~\forall x_1,x_2\in\mathbb{R}^{n_0}.\label{eq:WGAN_constraint}
\end{align}
\end{subequations}
Here, we maximize over the weights of the discriminator, i.e., the goal of $f_{\theta_1}$ is to distinguish data in $\mathcal{D}$ from artificially generated data $g_{\theta_2} (z)$, $z \sim \mathcal{N}(0,\mI)$. At the same time, we minimize over the weights of the generator, i.e., the generator should generate an output which is as indistinguishable as possible from the data in $\mathcal{D}$. Notice that the constraint~\eqref{eq:WGAN_constraint} specifies that $f_{\theta_1}$ must be $1$-Lipschitz continuous. This constraint \eqref{eq:WGAN_constraint} distinguishes Wasserstein GANs \cite{arjovsky2017wasserstein} from ordinary GANs \cite{goodfellow2014generative} and makes the training of the latter more stable. To guarantee Lipschitz continuity of the discriminator using \eqref{eq:lipschitz_matrix_inequality}, we solve the training problem
\begin{equation}\label{eq:WGAN}
    \begin{split}
    \min_{\theta_2} \max_{\theta_1,\kappa} ~& ~ \frac{1}{N} \sum_{x\in \mathcal{D}} f_{\theta_1}(x) - \mathbb{E}_{z\sim \mathcal{N}(0,\mI)} f_{\theta_1}(g_{\theta_2}(z))\\
    \mathrm{s.\,t.}~&~\mM(\theta_1,\kappa)\succeq 0,
    \end{split}
\end{equation}
where $\mM(\theta_1,\kappa) \succeq 0$ corresponds to the matrix inequality \eqref{eq:lipschitz_matrix_inequality} for Lipschitz continuity. Note that the discriminator contains convolutional layers that can be posed as feedforward layers using the following lemma \cite{sedghi2018singular}.
\begin{lemma}\label{lem:conv}
The linear transform for the convolution
\begin{align*}
    Y_{k,l} &= \sum_{i = 0}^{n} \sum_{j = 0}^{m} K_{i,j} X_{k+i,l+j} \quad \forall k,l.
\end{align*}
with filter $K$, input $X\in\bbR^{n\times m}$ and corresponding output $Y$ is expressed by the following doubly block circulant matrix
\begin{equation*}
    D=\begin{bmatrix}
        \circo(K_{0,:})   & \circo(K_{1,:}) & \dots  & \circo(K_{n-1,:})\\
        \circo(K_{n-1,:}) & \circo(K_{0,:}) & \dots  & \circo(K_{n-2,:})\\
        \vdots            & \vdots          & \vdots & \vdots\\
        \circo(K_{1,:})   & \circo(K_{2,:}) & \dots  & \circo(K_{0,:})\\
    \end{bmatrix}.
\end{equation*}
\end{lemma}
The output of the convolution can be respresented as $\veco(Y)=D\veco(X)$, where $\veco(\cdot)$ vectorizes the matrix argument. 
\begin{remark}
    For ease of exposition, we restrict ourselves to convolutions with one channel. The multi-channel case however carries over accordingly.
\end{remark}

\subsection{Convex relaxation of NNs using quadratic constraints}\label{sec:NN_description}
When designing the semidefinite constraints $\mM_j(\theta,\kappa)\succeq0$, $j=1,\dots,q$, we include information on the NN. In particular, we establish a convex relaxation of the NN, using the fact that the most common activation functions $\varphi:\mathbb{R}\to\mathbb{R}$, such as $\tanh$, sigmoid, and ReLU are slope-restricted.


\begin{definition}
A function $\varphi: \bbR \to \bbR$ is locally (globally) slope-restricted, $\varphi\in\slope[\alpha,\beta]$, if for all $x_1,x_2\in\mathcal{R}\subset\mathbb{R}$ ($x_1,x_2\in\mathbb{R}$)
\begin{equation*}\label{eq:slope_bounds}
    \alpha \leq \frac{\varphi(x_1)-\varphi(x_2)}{x_1-x_2} \leq \beta \quad \forall x_1\neq x_2 .
\end{equation*}
\end{definition}

Based on slope-restriction, we find a convex relaxation of the NN stated as an incremental QC. The key observation is that an NN with slope-restricted activation functions $\varphi\in\sec[\alpha,\beta]$ fulfills the incremental QC
\begin{equation}\label{eq:qc_slope}
    \begin{bmatrix}
        v_1-v_2 \\
        w_1-w_2
    \end{bmatrix}^\top
    \underbrace{\begin{bmatrix}
        2\alpha\beta\boldsymbol{\Lambda} & -(\alpha+\beta)\boldsymbol{\Lambda} \\
        -(\alpha+\beta)\boldsymbol{\Lambda} & 2\boldsymbol{\Lambda}
    \end{bmatrix}}_{=:\mM_\mathrm{NN}^\mathrm{slope}}
    \begin{bmatrix}
        v_1-v_2 \\
        w_1-w_2
    \end{bmatrix}\leq 0
\end{equation}
with multiplier matrix $\boldsymbol{\Lambda}=\blkdiag(\Lambda_1,\dots,\Lambda_l)\in\mathbb{D}^n_+$ \cite{fazlyab2020safety}, where $w_k = [{w_k^1}^\top, \dots, {w_k^l}^\top]^\top \in \mathbb{R}^n$, $k=1,2$ are instances of the outputs and $v_k\in\mathbb{R}^n,~k=1,2$ are instances of the inputs of the $n$ neurons such that $w_k=\phi(v_k)$ with $\phi(v_k)=[{\varphi(v_{k,1})},\dots,{\varphi(v_{k,n})}]^\top$.  


\subsection{Certification}
In Subsection \ref{sec:examples}, we addressed that Lipschitz continuity can be verified through a semidefinite constraint. In general, any behavioural property that is described by an incremental QC
\begin{equation}\label{eq:iQC}
     \begin{bmatrix}
         x_1-x_2 \\
         f_\theta(x_1) - f_\theta(x_2)
     \end{bmatrix}^\top
     \begin{bmatrix}
        \mQ & \mS\\
        \mS^\top & \mR
     \end{bmatrix}
     \begin{bmatrix}
         x_1-x_2 \\
         f_\theta(x_1) - f_\theta(x_2)
     \end{bmatrix}\geq 0,
\end{equation}
for all $x_1,x_2 \in \mathcal{R}^{n_0} \subseteq \mathbb{R}^{n_0}$, can be verified accordingly. 
Here, the incremental QC \eqref{eq:iQC} is characterized by the symmetric matrices $\mQ$ and $\mR$ and matrix $\mS$ of appropriate dimensions, where the choice $\mQ = L^2 \mI$, $\mR=-\mI$, $\mS=0$ describes $L$-Lipschitz continuity. To further formally state the certification condition, we introduce the transformation matrices 
\begin{equation}\label{eq:transformations}
    \begin{split}
        \begin{bmatrix}
            x_1-x_2 \\ f_\theta(x_1)-f_\theta(x_2)
        \end{bmatrix}=
        \underbrace{\left[ \begin{array}{c|cc}
            I & \multicolumn{2}{c} 0 \\ \hline
            0 &  0 & W_l
        \end{array}\right]}_{=:T_f}
        \begin{bmatrix}
            x_1-x_2 \\ w_1-w_2
        \end{bmatrix},\\
        \begin{bmatrix}
            v_1-v_2 \\ w_1-w_2
        \end{bmatrix}=
        \underbrace{
        \left[ \begin{array}{c|cc}
            W_0 & 0 & 0\\
            0 & N & 0\\\hline
            0 & \multicolumn{2}{c}I
        \end{array}\right]}_{=:T_\mathrm{NN}}
        \begin{bmatrix}
            x_1-x_2 \\ w_1-w_2
        \end{bmatrix},
    \end{split}
\end{equation}
where $N=\blkdiag(W_1,\dots,W_{l-1})$. Note that in the chosen incremental setup the bias terms cancel out and hence do not appear in~\eqref{eq:transformations}.

\begin{theorem}\label{thm:certificaion}
We consider an NN \eqref{eq:NN} with slope-restricted activation functions $\phi\in\slope[\alpha,\beta]$. If for some given matrices $\mQ\in\bbS^{n_0}$, $\mR\in\bbS^{n_{l+1}}$, $\mS\in\bbR^{n_0\times n_{l+1}}$, there exists  $\boldsymbol{\Lambda}\in\mathbb{D}_+^{n}$ such that
\begin{equation}\label{eq:matrix_inequality}
    T_f^\top      \begin{bmatrix}
        \mQ & \mS\\
        \mS^\top & \mR
     \end{bmatrix} T_f+ T_\mathrm{NN}^\top
     \mM_{\mathrm{NN}}^{\mathrm{slope}}
     T_\mathrm{NN}\succeq 0
\end{equation}
holds, then the property \eqref{eq:iQC} is satisfied for the NN \eqref{eq:NN}.
\end{theorem}
\begin{proof}
We left and right multiply (\ref{eq:matrix_inequality}) with $\begin{bmatrix}x_1^\top-x_2^\top & w_1^\top-w_2^\top\end{bmatrix}^\top$ and its transpose, respectively, and with \eqref{eq:transformations}, we obtain
\begin{equation}\label{eq:qc_sum}
\begin{split}
    \begin{bmatrix}
        x_1-x_2 \\ f_\theta(x_1)-f_\theta(x_2)
    \end{bmatrix}^\top
     \begin{bmatrix}
        \mQ & \mS\\
        \mS^\top & \mR
     \end{bmatrix}
    \begin{bmatrix}
        x_1-x_2 \\ f_\theta(x_1)-f_\theta(x_2)
    \end{bmatrix}\\
    +\begin{bmatrix}
        v_1-v_2 \\ w_1-w_2
    \end{bmatrix}^\top
    \mM_\mathrm{NN}^{\mathrm{slope}}
    \begin{bmatrix}
        v_1-v_2 \\ w_1-w_2
    \end{bmatrix}\geq 0.
    \end{split}
\end{equation}
Given that, using slope-restricted activation functions, \eqref{eq:qc_slope} holds by design of the NN, we add \eqref{eq:qc_slope} to \eqref{eq:qc_sum}, yielding \eqref{eq:iQC} which thus holds. 
\end{proof}
Note that the matrix inequality that verifies Lipschitz continuity \eqref{eq:lipschitz_matrix_inequality} is an instance of \eqref{eq:matrix_inequality} with $\mQ = L^2 \mI$, $\mR=-\mI$, $\mS=0$, $\alpha=0$, $\beta=1$, in addition requiring the application of the Schur complement for the convexification in $W_l$, cf. \cite{pauli2021training}. 


Based on Theorem~\ref{thm:certificaion}, we can enforce any desired property described by \eqref{eq:iQC} by including the corresponding matrix inequality constraint \eqref{eq:matrix_inequality} in the optimization problem \eqref{eq:training_problem} used for NN training.
\begin{remark}
Beside Lipschitz continuity, another relevant and interesting behavioural property clearly is stability of a dynamical system that includes an NN nonlinearity. Such stability constraints can be formulated as semidefinite constraints \cite{yin2021imitation}.
\end{remark}

\section{Training scheme}\label{sec:training}
In the following, we propose a training scheme to efficiently solve \eqref{eq:training_problem}, using the well-known log-det barrier function for $\mM_j(\theta,\kappa) \succ 0$, $j=1,\dots,q$ to transform \eqref{eq:training_problem} into the unconstrained optimization problem
\begin{equation}\label{eq:barrier}
\min_{\theta,\kappa}\, \mathcal{L}(f_\theta)-\sum_{j=0}^q\rho_j\log\det (\mM_j(\theta,\kappa))=\min_{\theta,\kappa}\, \mathcal{L}_\mM(\theta,\kappa),
\end{equation}
where $\rho_j>0$ are barrier parameters, that can be decreased for increasing iterations to gradually get a better approximation of the indicator function \cite{potra2000interior}. If the semidefinite constraint $\mM_j(\theta,\kappa)$ corresponds to the Lipschitz constraint \eqref{eq:lipschitz_matrix_inequality}, we exploit the banded structure of the matrix to accelerate the training, cf. Section \ref{sec:structure_exploit}.

The training based on \eqref{eq:barrier} can be carried out using backpropagation, where every first-order optimization method may be used to minimize the objective function $\mathcal{L}_{\mM}$. To this end, the gradients of $\mathcal{L}_{\mM}$ with respect to the decision variables $\theta$ and $\kappa$ can be determined analytically. Yet, we need to ensure initial feasibility as well as feasibility after every update step.
\begin{remark}
    The ADMM-based approaches suggested in \cite{pauli2021training,yin2021imitation} are restricted to LMIs, as the ADMM algorithm includes solving SDPs. Using barrier functions, as suggested in this paper, we may include nonlinear matrix inequality constraints, as well. Note that, e.g., the constraint \eqref{eq:lipschitz_matrix_inequality} is not jointly convex in the decision variables $\boldsymbol{\Lambda}$ and $\theta$.
\end{remark}

\subsection{Feasibility}\label{sec:feasible_initialization}
To start the training, the initial values $\theta_0$, $\kappa_0$, must satisfy the inequality constraints $\mM_j(\theta_0,\kappa_0)\succ~0$, $j=1,\dots,q$. For LMIs, a projection of the decision variables into the feasible set can be computed by solving an SDP. Yet, in the case that the semidefinite constraint is nonlinear in the decision variables, finding a feasible initialization is more challenging. Some applications, like Lipschitz-bounded NNs, admit intuitive ways for initialization.  To this end, we can exploit that given a fixed value of $L>0$, NNs with sufficiently small weights $W_i,~i=1,\dots,l$ are $L$-Lipschitz continuous.

To check feasibility of $\mM(\theta_{k+1},\kappa_{k+1})\succ 0$ after each update step, we make use of the following lemma.
\begin{lemma}[Cholesky decomposition]\label{lem:cholesky}
A matrix $\mX\!\in\!\mathbb{R}^{n\times n}$ admits a Cholesky factorization $\mX~\!\!=~\!\!\mL\mL^\top$ with a lower triangular, invertible matrix $\mL\in\mathbb{R}^{n\times n}$, if and only if $\mX$ is symmetric and positive definite. 
\end{lemma}
If the Cholesky decomposition of $\mM(\theta_{k+1},\kappa_{k+1})$ succeeds, then $\mM(\theta_{k+1},\kappa_{k+1}) \succ  0$ is feasible. Note that this feasibility check does not require additional computations since the Cholesky decomposition is also used to analytically determine the gradients for the gradient update steps.

\begin{remark}
    To achieve constraint satisfaction after an update step, in practice, we choose sufficiently small step sizes, that we decrease further for increasing iterations. Note that there are other computationally more expensive approaches to determine a step size that results in feasibility after an update step, e.g., line search methods to find the largest feasible step size. 
\end{remark}

\subsection{Training with Lipschitz bounds}\label{sec:structure_exploit}
In the following, we aim at solving the problem of training an NN with a bounded Lipschitz constant, enforcing the Lipschitz condition \eqref{eq:lipschitz_matrix_inequality} introduced in Subsection \ref{sec:ex_robust}. This training problem is taken from \cite{pauli2021training}, where it is solved using ADMM. We apply a more efficient training approach based on \eqref{eq:barrier}, where we in contrast to \cite{pauli2021training}, have the option to include the multiplier matrices $\Lambda_i$, $i=1,\dots,l$ as decision variables, yielding a bilinear matrix inequality. In addition, we exploit the block-banded structure of~\eqref{eq:lipschitz_matrix_inequality} to accelerate the training of the Lipschitz-bounded NN. The computational advantage becomes especially apparent in large-scale problems such as the training of WGANs, whose training problem we introduced in Subsection \ref{sec:ex_GANs}. 

To solve \eqref{eq:barrier} using backpropagation, we require the gradients of $\mathcal{L}_\mM$ with respect to $\theta$ and $\kappa$, wherein the gradient of $\psi(\mM) = \log\det\mM$ is given by $\nabla \psi(\mM) = \mM^{-1}$ \cite{boyd2004convex}. This means that, according to the chain rule, the gradients include the inverse of $\mM$ that we can determine efficiently using the Cholesky decomposition of $\mM=\mL\mL^\top$, yielding $\mM^{-1}=\mL^{-\top}\mL^{-1}$ for its inverse. We can further accelerate its computation exploiting the block-tridiagonal structure of~\eqref{eq:lipschitz_matrix_inequality}. We do this by finding the Cholesky decomposition
\begin{equation*}
    \begin{split}
        \mL=
        \begin{bmatrix}
            D_0 & 0        & \dots  &   0     \\
            R_0 & D_1      & \ddots & \vdots  \\
                &  \ddots  & \ddots & 0       \\
            0   &          & R_l    & D_{l+1}
        \end{bmatrix},
    \end{split}
\end{equation*}
where the individual blocks are computed by
\begin{equation*}\label{eq:triblock_cholesky}
    \begin{split}
        D_0 &= \chol( L^2 \mI ) = L \mI\\
        R_i &= - [ D_i^{-1} W_i^\top \Lambda_{i+1} ]^\top \quad {i=0,\cdots,l}  \\
        D_i &= \chol( 2\Lambda_i -  R_{i-1} R_{i-1}^\top ) \quad {i=1,\cdots,l} \\ 
        D_{l+1}  &= \chol( \mI - R_l R_l^\top ) ,
    \end{split}
\end{equation*}
from the individual blocks of~\eqref{eq:lipschitz_matrix_inequality}, where $\chol(\cdot)$ means the Cholesky decomposition. This way, we reduce the Cholesky decompotion of~\eqref{eq:lipschitz_matrix_inequality} to a sequence of $l+1$ Cholesky decompositons of smaller blocks, which is computationally favorable.

The inverse of a block-tridiagonal matrix, such as the matrix in~\eqref{eq:lipschitz_matrix_inequality}, is in general a full matrix. However, to compute the gradients, instead of computing the entire inverse $\mM^{-1}=\mL^{-\top}\mL^{-1}$, we only necessitate its block-tridiagonal entries which becomes apparent as follows. 
Let $\mM(x)=M_0+\sum_{i=1}^mx_iM_i$ where by $x_i\in\mathbb{R}$ we mean the $m$ scalar entries of $\mM$ that are functions of the decision variables $x=(\theta,\kappa)$, $M_i$ assigns them a position in $\mM$ and $M_0$ holds the constant terms of $\mM$. Using this notation, the gradient of $g(x_i) = \log \det(M_0+\sum_{i=1}^mx_iM_i)$ is $\nabla g(x_i) = \mathrm{tr}(M_i(M_0+\sum_{i=1}^m x_iM_i)^{-1})$, which implies that the non-zero entries in $M_i$ determine which terms of the inverse $\mM^{-1}$ are not cancelled out by the underyling matrix multiplication. According to \cite{asif2005block}, we set up
\begin{equation*}
    \begin{split}
        \mM^{-1}=
        \begin{bmatrix}
            S_0     & K_0^\top       & *        & \dots   &   *         \\
            K_0     & S_1            & K_1^\top & \ddots  & \vdots      \\
            *       & K_1            & S_2      & \ddots  & *           \\
            \vdots  & \ddots         &  \ddots  & \ddots  & K_l^\top    \\
            *       & \dots          & *        & K_l     & S_{l+1}
        \end{bmatrix},
    \end{split}
\end{equation*}
determining the relevant blocks of the inverse as follows
\begin{equation*} \label{eq:triblock_inverse}
    \begin{split}
       S_{l+1} & = ( D_{l+1} D_{l+1}^\top )^{-1} \\
       K_i &= - S_{i+1}^\top R_i D_i^{-1} \quad {i=l,\cdots,0} \\
       S_i &= ( D_i D_i^\top )^{-1} - K_i^\top R_i D_i^{-1}  \quad {i=l-1,\cdots,0}.
    \end{split}
\end{equation*}

Note that within this problem of training a Lipschitz-bounded NN based on the constraint \eqref{eq:lipschitz_matrix_inequality}, we can either treat the multiplier matrices $\Lambda_i$, $i=1,\dots,l$ (i) as decision variables or (ii) as hyperparameters that are set suitably before training, as done by \cite{pauli2021training}, which simplifies and convexifies the problem but introduces conservatism. The first option yields a bilinear matrix inequality (BMI) and the second a linear one. In the following, we distinguish the two variants by the names linear and bilinear barrier method. 
\begin{remark}
    In the bilinear case the semidefinite constraint becomes linear in the decision variables with a change of variables $(W_i, \Lambda_{i+1}) \to (\widetilde{W}_i,\Lambda_{i+1})$ in \eqref{eq:lipschitz_matrix_inequality}, where $\widetilde{W}_i := \Lambda_{i+1} W_i$, $i=0,\dots,l-1$. This way, additional nonlinearity is added to the NN.
\end{remark}

\section{Experiments}\label{sec:experiments}
In this section, we apply the presented method for NN training and illustrate its advantages. For the first example (Subsection \ref{sec:ex_2D}), we use a handcrafted optimizer programmed in C++ and for the MNIST (Subsection \ref{sec:ex_MNIST}), we use the machine learning library `Tensorflow' from Google. All simulations except the training of WGANs (Subsection \ref{sec:ex_WGANs}) are executed on an Intel i5-4670 with 16\,GB RAM\footnote{The code is available at \scriptsize{\url{https://github.com/eragon10/neural_network_training_with_matrix_inequality_constraints.git}} and \scriptsize{\url{https://github.com/eragon10/train-neural-networks-with-lipschitz-bound.git}}.}. The WGAN training is implemented in Pytorch and processed on the BwUniCluster where the NVIDIA Tesla V100 is used as a GPU for parallel computations. 

\begin{table}
    \centering
    \begin{tabular}{ l l | c c c c }
        \toprule
        &\textbf{method}                             & $L_\mathrm{mean}$ & $L_\mathrm{max}$ & Accuracy & $T$ \\
        \midrule
         \parbox[t]{2mm}{\multirow{5}{*}{\rotatebox[origin=c]{90}{2D example}}}
        &nominal                                     & $10780$ & $28197$ & $0.868$ & $1$ \\
        &projected                                   & $48.25$ & $49.54$ & $0.902$ & $7.35$ \\
        &ADMM \cite{boyd2004convex}                  & $46.08$ & $49.10$ & $0.851$ & $78.72$ \\ 
        &barrier (bilinear)                          & $43.94$ & $44.68$ & $0.903$ & $1.15$ \\
        &barrier (linear)                            & $47.99$ & $48.07$ & $0.904$ & $1.30$ \\
        \midrule
        \parbox[t]{2mm}{\multirow{4}{*}{\rotatebox[origin=c]{90}{MNIST}}}
        & nominal                                    & $597.13$ & $639.44 $ & $0.978$ & $1$ \\ 
        & barrier (bilinear)                         & $16.695$ & $16.761$ & $0.982$ & $1.74$ \\ \cmidrule{2-6}
        & nominal$^*$                                & $998.55$ & $1001.4$ & $0.962$ & $1$ \\
        & ADMM$^*$                                   & $7.802$ & $7.837$ & $0.883$ & $104.4$ \\ \bottomrule  
    \end{tabular}
    \caption{Simulation results for Lipschitz-bounded NNs with mean Lipschitz upper bound $L_\mathrm{mean}$ for 20 differently initialized NNs, maximum Lipschitz upper bound $L_\mathrm{max}$, accuracy on test data, and normalized mean training time~$T$. \hfill {\footnotesize{$^*$ based on \cite{pauli2021training}}}}
    \label{tab:simulation_results}
\end{table}

\subsection{Lipschitz-bounded NNs}
\subsubsection{Simple 2D example}\label{sec:ex_2D}
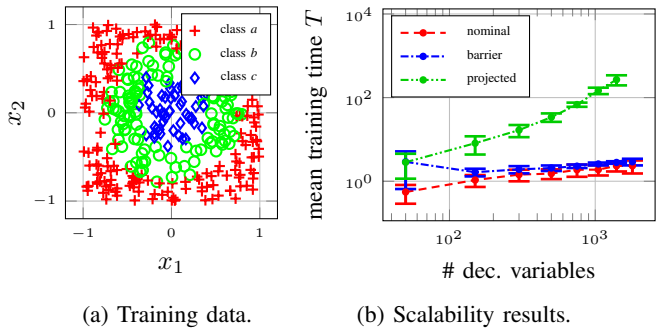
\begin{figure}
     \centering
     \begin{subfigure}[b]{0.18\textwidth}
         \centering
         \pgfplotsset{width=4.4cm,height=4.4cm,compat=1.8}


\hspace*{-0.75cm}
\begin{tikzpicture}[baseline]
    \pgfplotstableread{figures/training_data.csv}\sampledatatraining;
    

    
    \begin{axis}[legend style={legend columns=1,draw=black,anchor=west,align=center,
                        column sep=1cm,
                        /tikz/every odd column/.append style={column sep=0.2cm} },
            legend entries={ class \textit{a}, class \textit{b}, class \textit{c}},
            legend style={font=\tiny},
            ticklabel style={font=\tiny},
            enlarge x limits=true,
            enlarge y limits=true,
            trim axis right,
            trim axis left,
            enlargelimits,clip=false, 
            legend pos=north east,
            axis lines=box,
            grid=major,
            axis line style = {-latex},
            xlabel={$x_1$},
            ylabel={$x_2$},
            xmin=-1, xmax=1,
            ymin=-1, ymax=1,
            y label style={at={(-0.15,0.5)}},
        ]
        
            \addplot [
                scatter,only marks, scatter src=explicit symbolic,
                scatter/classes={
                    a={mark=+,red,thick},
                    b={mark=o,green,thick},
                    c={mark=diamond,blue,thick}
                }
            ] table[x=x,y=y,meta=label] {\sampledatatraining };
            
         \end{axis} 
    
%
%

   

\end{tikzpicture}
\vspace*{0.23em}

         \caption{Training data.}
         \label{fig:training_data}
     \end{subfigure}
     \begin{subfigure}[b]{0.25\textwidth}
         \centering
         \pgfplotsset{width=5.2cm,height=4.4cm,compat=1.7}



\begin{tikzpicture}[baseline]  
    
    \begin{axis} [ylabel={mean training time $T$},
                  xlabel={\# dec. variables},
                  enlarge x limits=true,
                  enlarge y limits=true,
                    trim axis right,
                trim axis left,
                ticklabel style={font=\tiny},
                label style={font=\small},
                 enlargelimits,clip=false, 
                  grid, ymode=log, ymax=5000, ymin=0.3,
                  legend cell align={left},
                  legend entries={nominal, barrier, projected},
                  legend pos=north west,
                  legend columns=1, xmode=log,
                  legend style={column sep=0.2cm, font=\tiny},
                  every error bar/.style={line width=0.4mm}]
                  
        \pgfplotstableread{figures/scale_nominell.csv}\timetablenominell;
        \addplot[red,thick, mark options={mark size=1pt, solid},] 
                    plot [densely dashed, mark=*, error bars/.cd, y dir=both, y explicit,
                            error bar style={solid},
                            error mark options={line width=1pt,mark size=4pt,rotate=90}] 
                    table [x=n, y=meantime, y error=sdtime]{\timetablenominell};
                    
        \pgfplotstableread{figures/scale_barriere.csv}\timetablebarriere;
        \addplot[blue,thick,mark options={mark size=1pt, solid},]
                    plot [densely dash dot, mark=*, error bars/.cd, y dir=both, y explicit,
                            error bar style={solid},
                            error mark options={line width=1pt,mark size=4pt,rotate=90}]
                    table [x=n, y=meantime, y error=sdtime]{\timetablebarriere};

                    
        \pgfplotstableread{figures/scale_projection.csv}\timetableprojection;
        \addplot[green!80!black,thick,mark options={mark size=1pt, solid},]
                    plot [densely dash dot dot, mark=*, error bars/.cd, y dir=both, y explicit,
                            error bar style={solid},
                            error mark options={line width=1pt,mark size=4pt,rotate=90}]
                    table [x=n, y=meantime, y error=sdtime]{\timetableprojection};
    
    \end{axis} 
    
\end{tikzpicture}

         \caption{Scalability results.}
         \label{fig:toy_scaleability}
     \end{subfigure}
     \caption{Training data and scalability results on 3-class 2D example for increasing number of decision variables.}
\end{figure}
We consider a simple 2D classification problem whose training data is shown in Fig.~\ref{fig:training_data}. On these data, we train feedforward NNs with two hidden layers of $10$ neurons each and activation function $\tanh\in\slope[0,1]$, using different training methods. We employ the optimizer ADAM \cite{kingma2014adam} and we enforce an upper bound on the NN's Lipschitz constant of $50$. In Table~\ref{tab:simulation_results}, we compare the barrier method proposed in this paper, wherein we exploit the structure of~\eqref{eq:lipschitz_matrix_inequality}, cf. Sect. \ref{sec:structure_exploit}, to (i) the ADMM approach suggested by \cite{pauli2021training}, and to (ii) projected gradient descent (PGD), i.e., the variables are projected into the feasible set after every iteration \cite{boyd2004convex,drummond2004projected}. 
Our approach achieves better accuracy on test data than the nominal NN and the NN trained using ADMM and it has shorter training times than the two alternative methods. In Fig.~\ref{fig:toy_scaleability}, we show NNs for increasing weight dimension. We see that, for this toy example, the nominal and barrier training have similar scalability whereas the PGD method scales orders of magnitude worse due to the bad scalability of the underlying SDP for projection.

\subsubsection{MNIST}\label{sec:ex_MNIST}
Next, we train a classifying NN on the MNIST dataset \cite{deng2012mnist}, using $14 \times 14$ pixel images as input data. Again, we compare our results to the ADMM approach from \cite{pauli2021training}, based on their original code. In our implementation of the barrier method, we include $\Lambda_i,~i=1\dots,l$ as decision variables, i.e., we use the bilinear barrier method. We use two hidden layers of 100 and 30 neurons, respectively, the activation function $\tanh$ and the optimizer ADAM. We set the upper bound of the Lipschitz constant to $20$. Table~\ref{tab:simulation_results} shows that the barrier method significantly outperforms the ADMM method with respect to training times and again achieves better accuracy on a test dataset. Note that, for comparibility, we show normalized training times for both methods (normalized to the nominal training for each framework (Tensorflow, PyTorch)).

\subsubsection{Wasserstein GANs}\label{sec:ex_WGANs}
\begin{table}
    \centering
    \begin{tabular}{ l | c c c}
        \toprule
        \textbf{method}   & $L_\mathrm{mean}$ & $L_\mathrm{max}$ & $T$ \\\midrule
        WC        & 0.0053 & 0.0061 & 1 \\ 
        GP        & 1,234 & 2,161 & 1.153\\  
        barrier (LMI) & 1.395 & 1.519 & 2.405\\
        barrier (BMI) & 1.381 & 1.495 & 2.419\\
        \bottomrule
    \end{tabular}
    \caption{Simulation results for WGAN training with mean Lipschitz upper bound $L_\mathrm{mean}$ for 5 differently initialized NNs, maximum Lipschitz upper bound $L_\mathrm{max}$, normalized mean training time~$T$. \label{tab:WGAN}}
\end{table}

\begin{figure}
     \centering
     \begin{subfigure}[b]{0.15\textwidth}
         \centering
         \includegraphics[width=\textwidth]{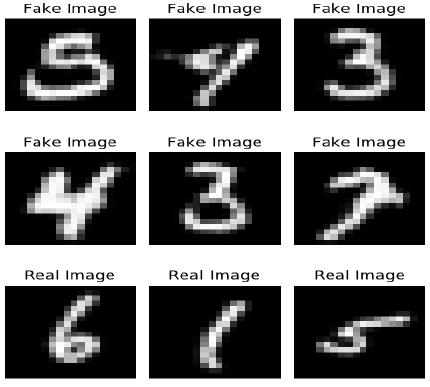}
         \caption{Matrix inequality (barrier).}
         \label{fig:barrier_lmi}
     \end{subfigure}
     \hfill
     \begin{subfigure}[b]{0.15\textwidth}
         \centering
         \includegraphics[width=\textwidth]{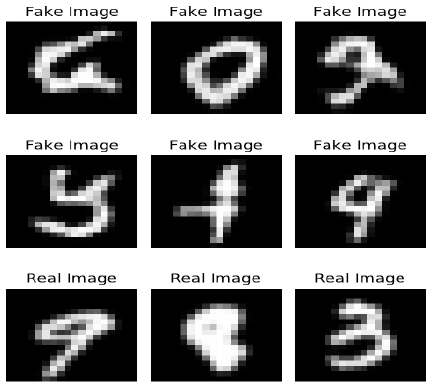}
         \caption{Weight clipping (WC).}
         \label{fig:weightclipping}
     \end{subfigure}
     \hfill
     \begin{subfigure}[b]{0.15\textwidth}
         \centering
         \includegraphics[width=\textwidth]{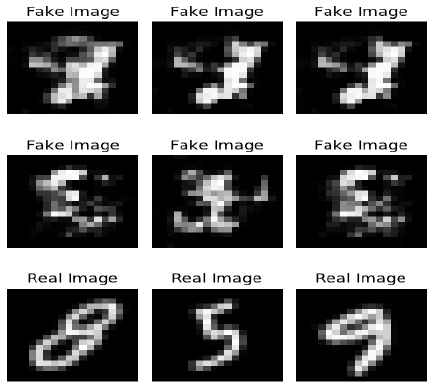}
         \caption{Gradient penalty (GP).}
         \label{fig:gradientnormpenalty}
     \end{subfigure}
        \caption{Simulation results and real and fake images generated by WGAN trained on MNIST data over 100 epochs each using different methods to enforce the $1$-Lipschitz constraint.}
        \label{fig:WGAN}
\end{figure}

While the training of robust NNs was addressed before in the literature \cite{pauli2021training}, the popular large-scale problem of WGAN training has not yet been implemented using semidefinite constraints. Our method provides an alternative to state-of-the-art methods for WGAN training, i.e., weight clipping \cite{arjovsky2017wasserstein} and gradient penalty \cite{gulrajani2017improved}. We compare the three methods in the following. To apply our training scheme to \eqref{eq:WGAN_prob}, we rewrite the convolutional layers in the discriminator as feedforward layers with sparse matrices according to Lemma~\ref{lem:conv}. To train the WGAN, we solve \eqref{eq:WGAN}, wherein we use the constraint \eqref{eq:lipschitz_matrix_inequality} to upper bound the Lipschitz constant of the discriminator by $1$, while exploiting the structure of \eqref{eq:lipschitz_matrix_inequality} according to Section~\ref{sec:structure_exploit} to accelerate the training. We train on $18\times 18$ images on the MNIST dataset \cite{deng2012mnist}, using the same architecture for all three methods and default and recommended hyperparameters. We select a fully connected and convolutional (FCC) WGAN architecture, where, for simplicity, in the generator and discriminator architectures, the bias terms are set to zero in all convolutional and deconvolutional layers. We choose an architecture with a 3-hidden-layer discriminator without batch normalization layers and a 5-hidden-layer generator architecture, where the generator activations are ReLU and $\tanh$ in the output layer and leaky ReLU with slope 0.2 and sigmoid in the discriminator, respectively.

As shown in Table \ref{tab:WGAN}\footnote{We state conservative upper bounds on the Lipschitz constant as solving an SDP based on \eqref{eq:lipschitz_matrix_inequality} is intractable for FCC-WGAN. We calculate the bounds based on \eqref{eq:lipschitz_matrix_inequality} with $\Lambda=\lambda \mI_n$ with sclalar decision variable $\lambda\geq 0$ and we split the NN in two and consequently multiply the Lipschitz bounds of the respective parts of the NN.}, weight clipping keeps the Lipschitz constant conservatively small whereas using a gradient penalty the estimated upper bound on the Lipschitz constant is much larger than $1$, which suggests that the Lipschitz condition may be violated. Thus, to effectively incorporate the constraints using the state-of-the-art methods, the hyperparameters have to be adjusted carefully. Our training approach based on a semidefinite constraint satisfies the Lipschitz constraint without any additional hyperparameter optimization. While an update step using our method is computationally more demanding, yet of the same order of magnitude as the other two, the single update steps are more efficient and therefore, training requires less epochs, cmp. the results after 100 epochs in Fig.~\ref{fig:WGAN}. In particular, in Fig.~\ref{fig:barrier_lmi}, we observe that the generator trained from the proposed barrier method produces fake images that look similar to the real data, whereas the other two methods in Figs.~\ref{fig:weightclipping} and \ref{fig:gradientnormpenalty} produce pictures that hardly look like digits from 0 to 9.

\section{Conclusions}
In this paper, we developed a training scheme for NNs considering semidefinite constraints. These constraints enforce desired properties onto the NN, such as Lipschitz continuity or closed-loop stability using an NN controller. To efficiently solve the underlying optimization problem, we developed an interior point method using logarithmic barrier functions and exploited the structure of the semidefinite constraints that enforce Lipschitz continuity. We illustrated the applicability to the training of classifying NNs with guaranteed Lipschitz bounds on the MNIST dataset and in addition, we trained WGANs guaranteeing with a semidefinite constraint that the discriminator NN is 1-Lipschitz.








\section*{ACKNOWLEDGMENT}
We acknowledge the support by the Stuttgart Center for Simulation Science (SimTech). The authors acknowledge support by the state of Baden-Württemberg through bwHPC. The authors thank the International Max Planck Research School for Intelligent Systems (IMPRS-IS) for supporting Patricia Pauli and Dennis Gramlich.


\bibliographystyle{IEEEtran}
\bibliography{references}

\end{document}